\documentclass[sigconf]{acmart}

\usepackage{booktabs} % For formal tables
 
\usepackage{graphicx} 
\usepackage{subcaption}
\usepackage[export]{adjustbox}
\setcopyright{none}
\renewcommand\footnotetextcopyrightpermission[1]{} % removes footnote with conference information in first column
\begin{document}
\pagestyle{empty}

\title{A State Transition Model for Mobile Notifications \\ via Survival Analysis}

\author{Yiping Yuan}
\affiliation{\institution{LinkedIn Corporation}}
\email{ypyuan@linkedin.com}

\author{Jing Zhang}
\affiliation{\institution{University of Maryland}}
\email{jzhang86@umd.edu}

\author{Shaunak Chatterjee}
\affiliation{\institution{LinkedIn Corporation}}
\email{shchatte@linkedin.com}

\author{Shipeng Yu}
\affiliation{\institution{LinkedIn Corporation}}
\email{siyu@linkedin.com}

\author{Romer Rosales}
\affiliation{\institution{LinkedIn Corporation}}
\email{rrosales@linkedin.com}

\renewcommand{\shortauthors}{Y. Yuan et al.}

\begin{abstract}
Mobile notifications have become a major communication channel for social networking services to keep users informed and engaged. As more mobile applications push notifications to users, they constantly face decisions on what to send, when and how. A lack of research and methodology commonly leads to heuristic decision making. Many notifications arrive at an inappropriate moment or introduce too many interruptions, failing to provide value to users and  spurring users' complaints. In this paper we  explore unique features of  interactions between mobile notifications and user engagement. We propose a state transition framework to quantitatively evaluate the effectiveness of notifications. Within this framework, we develop a survival model for badging notifications assuming a log-linear structure and a Weibull distribution. Our results show that this model achieves more flexibility for applications and  superior prediction accuracy than a logistic regression model. In particular, we provide an online use case on notification delivery time optimization to show how we make better decisions, drive more user engagement, and provide more value to users.

\end{abstract}

\begin{CCSXML}
<ccs2012>
 <concept>
  <concept_id>10010520.10010553.10010562</concept_id>
  <concept_desc>Computer systems organization~Embedded systems</concept_desc>
  <concept_significance>500</concept_significance>
 </concept>
 <concept>
  <concept_id>10010520.10010575.10010755</concept_id>
  <concept_desc>Computer systems organization~Redundancy</concept_desc>
  <concept_significance>300</concept_significance>
 </concept>
 <concept>
  <concept_id>10010520.10010553.10010554</concept_id>
  <concept_desc>Computer systems organization~Robotics</concept_desc>
  <concept_significance>100</concept_significance>
 </concept>
 <concept>
  <concept_id>10003033.10003083.10003095</concept_id>
  <concept_desc>
  s~Network reliability</concept_desc>
  <concept_significance>100</concept_significance>
 </concept>
</ccs2012>  
\end{CCSXML}

%\ccsdesc[500]{Computer systems organization~Embedded systems}
\ccsdesc{Information systems~Data mining}
\ccsdesc{Mathematics of computing~Survival analysis}

\keywords{Mobile notifications; survival analysis; Weibull distribution;  accelerated failure-time model}

\maketitle

%\input{body}
%%%%%%%%%%%%%%%%%%%%%%%%%%%
\section{Introduction}
\label{sec:introduction}
Social networking services (e.g., Facebook, LinkedIn, Instagram, Twitter, WeChat) actively push information to their users through mobile notifications.  As the content ecosystem and users' connection networks grow, more and more information is generated on the social networking site that is worth  informing the users. On the other hand, users have limited attention span, regardless of how much value notifications could inform them of. The discrepancy between increasing content and limited user attention is the challenge many mobile applications are facing, especially those social networking applications. 

A mobile notification is a message displayed to the user either through the mobile application \textit{user interface} (UI) itself, or through the operating system's push notification services, such as \textit{Apple Push Notification Service} (APNs). Instances of such messages include  a user-to-user communication, a friend or connection request, an update from a friend or connection (e.g., birthday or job change), an article posted by a connection, etc. These notifications help keep the users informed of what is happening in their network. In addition, notifications also serve the purpose of promotions and product marketing for many mobile applications.

Compared with email communication, mobile notifications are more time sensitive and more promptly responded to \cite{fischer2011investigating,pielot2014situ}. Without an established way to determine delivery time, mobile notifications often arrive at inconvenient moments, failing to provide value to a user. Moreover, due to the pervasive nature of smartphones, such inconvenience may lead to complaints or even disablement on future notification deliveries, causing a permanent loss to both service providers and users. In short, sending notifications at the right time with the right content in many cases is critical.

%Specifically, we are interested in the following questions,
% \begin{itemize}
%\item How effective the mobile notifications promote app visits. Since users may visit the site regardless of notifications, the challenge is to learn the incremental effect of each notification.
%\item How the effect differs from user to user. This personalization, a typical machine learning problem, is important to our decision making.  
%\item Whether the effect is time sensitive. If so, we may be able to maximize the effect by choosing the right delivery time with the same set of notifications.
%\end{itemize}

In this paper, we focus on a quantitative way to measure the effectiveness of a mobile notification and to learn the pattern of how the effectiveness differs from  user to user and from time to time. The overall objective is to improve personalization and ensure better delivery time and volume optimization.

The interaction of a user with mobile notifications can be very complex and depends on numerous aspects \cite{mehrotra2015designing, mehrotra2014sensocial, xu2013preference}. It is common to link a notification event to one or more rewards to evaluate the effectiveness of a notification. For engagement, a typical reward is a visit from the user to the app.  One challenge for such a study is how to attribute a reward, because users may receive multiple notifications before they open and visit the app. Simple strategies could be to attribute the reward to the most recent one or to several notifications delivered within a look-back time period. Such strategies are hard to justify theoretically and could introduce significant bias in learning. Our strategy is to leverage the survival analysis to attribute a reward without ambiguity and bias \cite{buckley1979linear,james1984consistency}. %As a notification publisher, it is often necessary to make simplifications and assumptions about how to study this problem. 

Survival analysis is commonly used within medical and epidemiological research to analyze data where the outcome variable is the time until the occurrence of an event of interest. For example, if the event of interest is heart attack, then the time to event or survival time can be the time in years until a person develops a heart attack. In survival analysis, subjects are usually followed over a specified time period and the focus is on the time at which the event of interest occurs. Survival time has two components that must be clearly defined: a beginning point and an endpoint that is reached either when the event occurs or when the follow-up time has ended. If the event does not occur by the follow-up time, the observation is called censored. The censored observations are known to have a certain amount of time where the event of interest did not occur and it is not clear whether the event would have occurred if the follow-up time were longer. Such censoring is very common in observational notification data.

We introduce survival analysis to notification modeling as a new domain. The beginning point in this case is the delivery time of a notification and the endpoint is the reward time (e.g., the time of a visit) or a next notification delivery time, whichever happens first. When the next notification occurs first, the observation is censored. In this paper, we apply an accelerated failure-time model \cite{wei1992accelerated,keiding1997role} with a Weibull distribution to our large-scale user data for the reward prediction. This turns out to be not just novel, but also superior in prediction performance compared to baseline models in our offline analysis.  

We provide two example formulations for notification volume optimization (VO) and delivery time optimization (DTO) separately. We then present an online use case on notification DTO, where our model is used to make send decisions. The A/B test results show significant improvement on user engagement and content consumption over a non-DTO control and a baseline DTO model.

The major contributions of this paper can be summarized as follows.
 \begin{itemize}
 \item We develop a state transition model to measure the effectiveness of a notification through a  delta effect $\Delta F(W_0,T)$ in Section \ref{sec:main_model}. 
 \item We propose to estimate the delta effect in the presence of censored data, using a log linear survival structure and a Weibull distribution.
 \item We conduct offline evaluations with real-world notification data to
demonstrate the accuracy and flexibility of our engagement prediction.
 \item We carry out an online use case of determining the delivery time and show superiority of the proposed method with A/B tests in Section \ref{sec:usecase}.  
 \end{itemize}

\section{Related work}
\label{sec:related}

Email communication as a channel has a long history for social networking services. A volume optimization framework \cite{gupta2016email, gupta2017optimizing} can simultaneously minimize the number of emails sent, control the negative complaints, and maximize user engagement. While we share similar goals for mobile notifications, there are unique mobile aspects to be considered. Moreover, the volume optimization framework focuses on solving a Multi-Objective Optimization (MOO) problem \cite{agarwal2011click, agarwal2012personalized}, in which multiple objectives are optimized under given constraints. We focus on probabilistic nature of interactions between a user and a notification. Our work can be leveraged as a utility prediction model, which would be one of the utilities of interest in a MOO formulation for mobile notifications.

As more mobile applications push information to users, several studies have been carried out to understand how to make effective use of notifications. 
Sahami et al. \cite{sahami2014large} collect close to 200 million notifications
from more than 40,000 users including users' subjective
perceptions and present the first large-scale
analysis of mobile notifications. A number of findings about the nature of notifications, such as shorter responding time, have shed light on how to effectively use them.  Pielot et al. \cite{pielot2014situ} carry out an one-week, in-situ study involving 15 mobile phones users and suggested that  an increase in the number of notifications is associated with an increase in
negative emotions.  Both works do not attempt to model the interactions probabilistically.

Xu et al. \cite{xu2013preference} developed an app usage prediction model that leverages the user's day-to-day
activities, app preferences and  the surrounding environment. Mehrotra et al. \cite{mehrotra2015designing} developed a classification model to predict notification acceptance by considering both content and context information.  Pielot et al. \cite{pielot2014didn} proposed  a machine learning model  to predict whether the user will view a message within the next few minutes or not after a notification is delivered. Their study also suggests that indicators of availability, such as the last time the user has been online, not only create social pressure, but are also weak predictors of attentiveness of the message. Pielot et al. \cite{pielot2017beyond} carried out a field study with hundreds of mobile phone users and built a machine-learning model to predict whether a user will click
on the notification and subsequently engage with the content. The model can be used to determine the opportune moments to send notifications. These studies focus on cross-application study with  complete device information, yet the scale of  notifications and users are not comparable to our case.

On general user engagement,  extensive studies \cite{attenberg2009modeling, dave2014computational, wang2013psychological, ashley2015creative, khan2017social} have been promoting relevant and high quality content to users 
to maximize long-term user engagement with the platform. Other works \cite{goldstein2013cost, yoo2005processing} show that  low-quality advertising has detrimental effect on long-term user engagement.  Zhou et al. \cite{zhou2016predicting} developed an ad quality model based on logistic regression to identify offensive ads that affect user engagement. The focus has been on the quality instead of the timing.

Most applications of the survival analysis in the literature have  been in  medicine, biology or public health, but there is an increasing interest in its applications to social behavior. Survival techniques based on Weibull distributions have been applied to understanding and predicting dwell time on web services \cite{liu2010understanding, Vasiloudis:2017:PSL:3077136.3080695}. Yu et al. \cite{yu2017temporally} proposed a temporally heterogeneous survival framework to model social behavior dynamics, whose model parameters can be solved by maximum likelihood estimation. The model can be applied to user-to-user communication. Gomez-Rodriguez et al. \cite{gomez2013modeling}   studied the formation of an information cascade in a network based on survival theory. Last but not least, Li et al. \cite{li2017prospecting} applied survival analysis in modeling the career paths. They formulated the problem as a multi-task learning and achieved favorable performance against several other state-of-the-art machine learning methods.

\section{State transition model via survival analysis}
\label{sec:main_model}
 
A mobile notification may be delivered through one or many channels such as a sound, a badge count update on the app icon, and an alert shown on the lock screen or as banners. A UI push notification shown in Figure \ref{fig:push-notification} refers to one with an alert showing the content of the message. Such notifications are more effective at drawing a user's attention but they can also be intrusive or even annoying. As suggested in studies \cite{sahami2014large,pielot2014situ}, the UI push channel  is better for time-sensitive and potentially important notifications, e.g., a connection invitation or a user-to-user message. Other less time-sensitive ones, e.g., a connection's birthday or work anniversary, can be served as badging notifications, meaning we only push a badge count update as shown in Figure \ref{fig:badgecount} and a user has to open the app to see the content as an in-app notification within the mobile application's UI in Figure \ref{fig:inapp-notification}. Unlike UI push notifications, such badging notifications are much less intrusive. On the other hand, the effect of them are more subtle. Users are not able to view and interact with the notification content without opening the app. It usually takes longer time for a user to respond to the badging than the UI push and it is harder to separate the effect of notifications from the organic visits. Attribution challenges also arise when multiple badging notifications have been delivered with more than one badge count. This challenge is further elaborated in Section \ref{sec:time-to-visit} as data censoring. The content of badging notifications are usually less time-sensitive and hence we have more flexibility in their delivery time. 

\begin{figure}[h]
\includegraphics[width=0.6\linewidth]{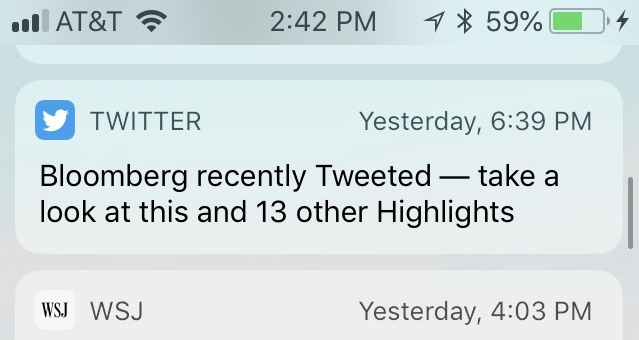}
\caption{\label{fig:push-notification} An example of UI push notifications } 
\end{figure} 

\begin{figure}[h]
\begin{adjustbox}{minipage=\linewidth,scale=0.6}
\begin{subfigure}{0.32\textwidth}
\includegraphics[width=\linewidth]{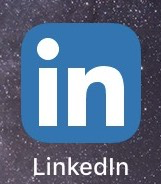}
 \label{fig:bc0}
\end{subfigure}\hspace*{\fill}
\begin{subfigure}{0.32\textwidth}
\includegraphics[width=\linewidth]{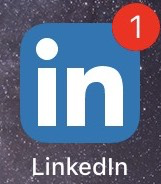}
 \label{fig:bc1}
\end{subfigure}\hspace*{\fill}
\begin{subfigure}{0.32\textwidth}
\includegraphics[width=\linewidth]{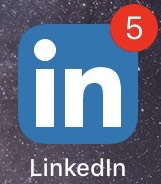}
 \label{fig:bc5}
\end{subfigure}
\end{adjustbox}
\caption{Visual appearance of badge counts} \label{fig:badgecount}
\end{figure}

\begin{figure}[h]
\includegraphics[width=0.6\linewidth]{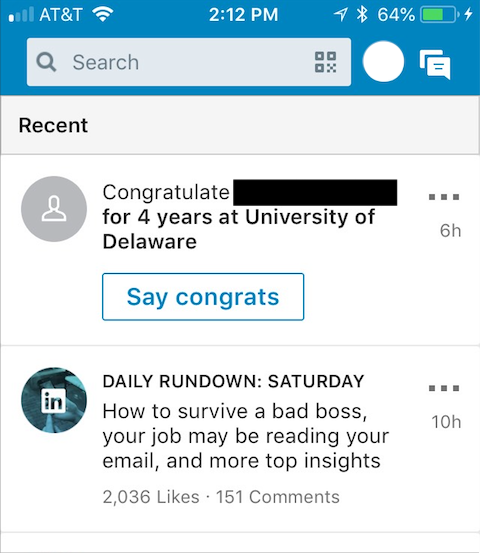}
\caption{\label{fig:inapp-notification} An example of in-app notification } 
\end{figure} 
In this section, we develop a state transition model to describe and predict the interactions between users and notifications. We  focus on badging notifications, which has a more subtle effect to model and is less studied in the literature.  The methodology can be extended to UI push notifications with possibly different distribution assumptions as they are responded to more quickly.  

%Different channels they use, we view both as interventions to  users' behavior. Some notifications have to be delivered instantly, for example, if a user get a message from another user, we want to the user to be immediately notified. In this case, we want to understand how notifications as interventions change the user's visiting time of the app. In other cases, notifications are not time sensitive and we have more control over when to send. One example would be a connection of a user is having birthday on a date, we can decide how much in advance we would like to inform the user. We are not only interested in understanding the dynamics, but also how to maximize the interventional effect of notification with better timing.

\subsection{State Transition Model}
\label{sec:statetransition}
We aim to learn how notifications as interventions promote user engagement and bring more value to users. Notifications may change users' mobile context state in various ways. For badging notifications, the most obvious one would be the change of the outer badge count. They may also change the notification inventory within the app.   

Let $M$ be a notification event, $s$ be a mobile context state, and $t_{s}$ be the time to the next visit since the start time of the state $s$. Figure \ref{fig:state-transition} shows how a state transition model works. After a notification $M_0$, a mobile context state stays at $s_{0}$. Then at any evaluation time, we consider whether or not to send a notification $M_1$ to a user, who has stayed in state $s_0$ for $W_0$ time. The mobile context state will change to $s_{1}$ if $M_{1}$ is received. Note that a user's visit can also change the state, so $s_0$ may start from the most recent visit event or the most recent notification event, whichever comes later.

\begin{figure}[h]
\includegraphics[width=\linewidth]{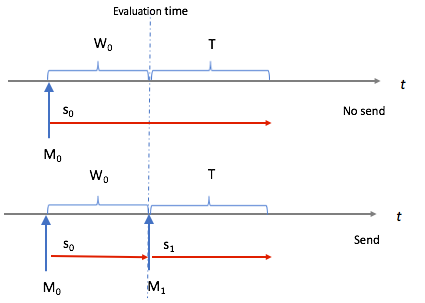}
\caption{\label{fig:state-transition} Illustration of state transition} 
\end{figure}

In our state transition model, we assume that users' engagement behaviors depend on both their mobile context states and users' characteristics. If $M_1$ is sent, then state $s_1$ kicks in and the probability of a user visiting our app within the next $T$ time would be 
\begin{equation}
\label{eq:pvisit}
P(\text{visit|send}) = \Pr(t_{s_1}< T \mid \boldsymbol{z},  s_{1}) = F_{t_{s_1} \mid \boldsymbol{z},  s_{1}}(T ),
\end{equation}
where $s_{1}$ represents this user's new mobile context state after the notification is sent, $\boldsymbol{z}$ represent this user's features and $F_{t_{s_1} \mid \boldsymbol{z},  s_{1}}$ is the cumulative distribution function of time-to-visit $t_{s_1}$ given $(\boldsymbol{z},  s_{1})$. $T$ is the prediction window whose value is usually chosen based on the specific problem instance. For example, we can set it to be 24 hours if we want to focus on daily active users.

%
%When we consider the interventional effect of a notification, we can compare the following two probabilities,
%
%$$\Pr(t \leq T\mid \boldsymbol{z}, m), $$
%the probability of the user  visiting the site/app in the next $T$ time window given a notification $m$ sent, and 
%
%$$ \Pr(t \leq T\mid \boldsymbol{z}, \neg m),$$
%the probability of the user  visiting the site/app in the next $T$ time window without a notification sent. if the difference is large, then there is a strong motivation for us to send a notification. We choose $T$ depending our use cases. We could set it to be one day if we want to boost daily active users.
%
%In the state transition model, the first probability can be translated as 

If we decide not to send a notification $M_1$, the user will stay in the current state $s_{0}$. Then the probability of the next visit within the next $T$ time is 
\begin{align}
P(\text{visit|not send})=&\Pr (t_{s_0} \leq T +  W_{0} |\boldsymbol{z},  s_{0}, t_{s_0} >  W_{0}  ) \nonumber\\
 = &\frac{\Pr (W_{0} < t_{s_0} \leq T +  W_{0} |\boldsymbol{z},  s_{0} ) } { \Pr ( t_{s_0} > W_{0} |\boldsymbol{z},  s_{0}) }  \nonumber\\
 = & \frac{ F_{t_{s_0} \mid \boldsymbol{z},  s_{0}}(T +  W_{0} ) -  F_{t_{s_0} \mid \boldsymbol{z},  s_{0}}(W_{0} )} {1 - F_{t_{s_0} \mid \boldsymbol{z},  s_{0}}( W_{0} )}, \label{eq:pvisitneg}
\end{align}
which is the probability of time-to-visit from the last state $t_{s_0}$ being less than or equal to $T+W_0$ given that $t_{s_0}$ is already greater than $W_0$.

We name the difference between \eqref{eq:pvisit} and \eqref{eq:pvisitneg} the delta effect, which is a function of $T$ and $W_0$ given $\boldsymbol{z}$, $s_0$ and $s_1$,
\begin{align}
\Delta F(W_0,T)&=&& P(\text{visit|send}) -P(\text{visit|not send})  \nonumber\\
&=& &F_{t_{s_1} \mid \boldsymbol{z},  s_{1}}(T ) -\frac{ F_{t_{s_0} \mid \boldsymbol{z},  s_{0}}(T +  W_{0} ) -  F_{t_{s_0} \mid \boldsymbol{z},  s_{0}}(W_{0} )} {1 - F_{t_{s_0} \mid \boldsymbol{z},  s_{0}}( W_{0} )}  .
\label{eq:deltapvisit}
\end{align}
The delta effect predicts the additional probability of visit in the next $T$ time by sending a notification at the moment. The larger the delta effect is, the more motivation we have to deliver a notification. 
 
In \eqref{eq:deltapvisit}, we need to learn  the distribution of users' time to visit in each state $ F_{t \mid \boldsymbol{z},  s}( T)$ to predict the delta effect. We  explain how we estimate this distribution in Section \ref{sec:time-to-visit}.

%%%%%%%%%%%%%%%%%%%%%%%%%%%%%%%%%%%%%%%%%%%%%%%%%%%%%%%%%%

\subsection{Time-to-visit Forecasting}
\label{sec:time-to-visit}

One of the challenges for learning $ F_{t \mid \boldsymbol{z},  s}( T)$ is  that we can not always observe the time to visit after a notification send event, because we may send out another notification before the user's next visit.  Figure \ref{fig:censoring} illustrates the mobile activities of a user. After $T_1$ with notification event $M_1$, we observe a visit $V_1$. And $T_4$ after notification event $M_4$ we observe a visit $V_2$. We do not observe a visit after $M_2$ and $M_3$ before their next notification events $M_3$ and $M_4$, respectively. In the latter cases, the two observations are censored. A censored observation only tells you that the visit event has not happened before the next notification arrives. Such censored observations are very common in notification training data, especially for less active users since the average time-to-visit after a notification delivery is longer.

\begin{figure}[h]
\includegraphics[width=\linewidth]{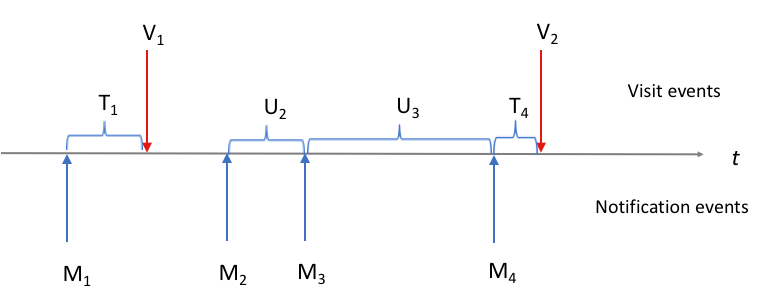}
\caption{\label{fig:censoring} Right-censoring} 
\end{figure}

Therefore, we observe either a visit time $T_i$, or a censored time $U_i$. An observation in survival analysis can be
conveniently represented by a triplet $(X_i; T_i; \delta_i)$. Here $X_i$ is a feature vector containing both users' features $\boldsymbol{z}$ and state features  $s$; $\delta_i$ is the censoring indicator, specifically, $\delta_i= 1 $ for an uncensored instance and $\delta_i = 0$ for a
censored instance; and $T_i$ denotes the observed visit time if $\delta_i = 1 $ and a censoring time if  $\delta_i = 0$.

While it is possible to avoid or alleviate such censoring by collecting data from a controlled experiment, we argue that it is impractical in many cases. For example, in a controlled study we send a mobile notification to every user in the treatment group at the beginning of the experiment, and then monitor the next visit event without sending more notifications in-between.  First, the treated users may get very negative user experience without being promptly notified. Secondly, the experiment may take a long time to observe a visit event for less active users. Lastly, it becomes too costly to repeat the experiment frequently when the model has to be re-trained over time with updated user bases and features.

Survival methods correctly incorporate information from both
censored and uncensored observations  for estimation, through maximizing the following likelihood function

\begin{align}
L=\prod_{i: \delta_i=1} f(t = T_i \mid X_i) \prod_{i: \delta_i=0} \Pr (t > T_i \mid X_i)\nonumber\\
= \prod_{i=1}^{n} \left(f_{t \mid X_i}( T_i)\right)^{\delta_i} \left((1-  F_{t \mid X_i}( T_i))\right)^{1-\delta_i},
\label{eq:survival-likelihood}
\end{align}
where $f_{t \mid X_i}$ is the probability density function.

A well-known survival model is the Cox proportional hazards model \cite{cox1992regression,kapoor2014hazard}. It is a semi-parametric model built  upon the assumption of
proportional hazards. In other words, it assumes that the effects of the predictor variables on survival are constant over time and are additive in one scale. This assumption may not be realistic for our application. In addition, a nonparametric baseline hazard function from the Cox model is difficult to interpret and to conduct statistical inference with. A popular alternative survival model is the parametric log-linear model, which is also known as accelerated failure-time (AFT) model \cite{wei1992accelerated,keiding1997role}. In this model, the effect of changing a covariate is to accelerate or decelerate the time-to-event by some factor. The parametric form makes it much easier to evaluate $F(W_0,T)$ in \eqref{eq:deltapvisit}. In addition, a property in Lemma \ref{lemma:decreasing} works well in practice to space notifications. Therefore,  we  use the AFT model for our time-to-event forecasting. 

The AFT model proposes the following relationship between a random time-to-visit $T_i$ and covariates $\boldsymbol{X}_i$,
\begin{equation}
\label{eq:aft}
\log T_i =  \boldsymbol{b} \boldsymbol{X}_i + \sigma \epsilon_i,
\end{equation}
where $\epsilon_i$ are \textit{independent and identically distributed} (i.i.d.) random errors.

Popular distributions for $\epsilon_i$ are logistic, Gaussian and extreme value distributions, leading to log-logistic, log-Gaussian, and Weibull distributions for $T_i$, respectively. Based on our data analysis and prior knowledge, the time-to-visit for badging notifications given the users' features and state features does not quite depend on how much time has elapsed already, which is the memoryless property. The distribution may be close to an exponential distribution, a special case of the Weibull distribution. Therefore, we assume a Weibull distribution for $t\mid\boldsymbol{z},  s$. The Weibull distribution is a flexible model for time-to-event data \cite{klein2005survival}. The probability density function and the cumulative distribution function of Weibull distribution are

\begin{equation}
 f(T;\lambda,\alpha) =\begin{cases}\alpha\lambda T^{\alpha-1}e^{-\lambda T^{\alpha}} & T\geq0 ,\\0 & T<0,\end{cases} 
 \end{equation}
 and 
\begin{equation}
\label{eq:weibullcdf}
F(T;\lambda, \alpha) = \Pr(t \leq T) =   1 - e^{-\lambda T^\alpha} \quad T\geq0.
\end{equation}
The exponential distribution is a special case when $\alpha = 1$.

Assume $\epsilon_i$ in \eqref{eq:aft} follows an extreme value distribution with

\begin{equation}
\label{eq:extremevalue}
 f_{\epsilon}(t) = e^{(t-e^t)}, F_{\epsilon}(t) = 1- e^{-e^t},
\end{equation}
then $T_i$ follows Weibull distribution \cite{klein2005survival} with 
\begin{equation}
\label{eq:lambda_alpha}
\lambda_i = e^{-\mu_i/\sigma}; \quad\alpha = 1/\sigma,
\end{equation}
where $\mu_i =  \boldsymbol{b} \boldsymbol{X}_i$. Note that the model assumes no heteroscedasticity for simplicity, which implies that $\sigma$ as well as $\alpha$ are constants. It is possible to assume heteroscedasticity and model $\sigma$ as a function of features $\boldsymbol{X}_i$, adding more personalization in estimating the distribution of time-to-visit $T_i$ for different users at different states. On the other hand, the maximum likelihood estimation is going to be more computationally challenging.

\begin{lemma}
\label{lemma:decreasing}
If $t_{s_0} \mid \boldsymbol{z},  s_{0}$ follows a Weibull distribution $ f(T;\lambda_0,\alpha_0) $ with $\alpha_0 \in (0,1)$, then $P(\text{visit|not send})$ in \eqref{eq:pvisitneg} is decreasing and thus $\Delta F(W_0,T)$ in \eqref{eq:deltapvisit} is increasing in $W_0$, the time that has elapsed already in state $s_0$, for any given $T > 0$.
\end{lemma}

\begin{proof}
With the Weibull distribution in \eqref{eq:weibullcdf}, $\Delta F(W_0,T)$ in \eqref{eq:deltapvisit} becomes
\begin{align}
&\Delta F(W_0,T)\nonumber\\
 = &  F_{t_{s_1} \mid \boldsymbol{z},  s_{1}}(T ) -\frac{ F_{t_{s_0} \mid \boldsymbol{z},  s_{0}}(T +  W_{0} ) -  F_{t_{s_0} \mid \boldsymbol{z},  s_{0}}(W_{0} )} {1 - F_{t_{s_0} \mid \boldsymbol{z},  s_{0}}( W_{0} )}  \nonumber\\
 = & 1 - e^{-\lambda_1 ( T)^{\alpha_1}} - \frac{\{1 - e^{-\lambda_0 (T + W_{0})^{\alpha_0}}\} - \{1 - e^{-\lambda_0 ( W_{0})^{\alpha_0}}\}}{ e^{-\lambda_0 ( W_{0})^{\alpha_0}} }\nonumber\\
= &  e^{-\lambda_0 (T + W_{0})^{\alpha_0} + \lambda_0 (W_{0})^{\alpha_0} } - e^{-\lambda_1 ( T)^{\alpha_1}}.\nonumber
 \end{align}

Taking derivative with respect to $W_{0}$,

\begin{align}
& \frac{\partial \Delta F(W_{0},T)}{\partial W_{0}}  \nonumber\\
= &  e^{-\lambda_0 (T + W_{0})^{\alpha_0} + \lambda_0 (W_{0})^{\alpha_0} }\{ -\lambda_0 \alpha_0 (T + W_{0})^{\alpha_0 -1} + \lambda_0 \alpha_0(W_{0})^{\alpha_0-1}\}. \nonumber
 \end{align}
  
Since $(T + W_{0})^{\alpha_0 -1} < (W_{0})^{\alpha_0-1}$ for  $\alpha_0 \in (0,1)$ and $T>0$. Then we have $ \frac{\partial\Delta F(W_{0},T)}{\partial W_{0}} > 0$ for  $\alpha_0 \in (0,1)$ and $T>0$.
\end{proof}

Lemma \ref{lemma:decreasing} shows that if  $t_{s_0} \mid \boldsymbol{z},  s_{0}$  follows a Weibull distribution with  $\alpha_0 \in (0,1)$, the delta effect in \eqref{eq:deltapvisit} can be calculated as 

\begin{equation}
\label{eq:weibulldeltapvisit}
\Delta F(W_{0},T)  =  e^{-\lambda_0 (T + W_{0})^{\alpha_0} + \lambda_0 (W_{0})^{\alpha_0} } - e^{-\lambda_1 T^{\alpha_1}} ,
\end{equation}
and is increasing in $W_{0}$. This suggests we can bring more value to a user by sending a notification when the user has stayed in a state for a longer time. In other words, incorporating the delta effect into decision making reduces the frequency of sending a notification, because short intervals between notifications do not engage the user's attention effectively. The model we learned from data in Section \ref{sec:datacollection} gives  $\alpha \in (0,1)$, which is in line with our conjecture.

Following \eqref{eq:survival-likelihood}, the likelihood function becomes

\begin{align}
L= \prod_{i=1}^{n} \left(f_{\epsilon}\left(\frac{\log T_i - \boldsymbol{b} \boldsymbol{X}_i}{\sigma}\right)\right)^{\delta_i} \left((1-  F_{\epsilon}\left( \frac{ \log T_i - \boldsymbol{b} \boldsymbol{X}_i}{\sigma}\right))\right)^{1-\delta_i},
\label{eq:aft-likelihood}
\end{align}
where $f_{\epsilon}$ and $F_{\epsilon}$ are from the extreme value distribution as in  \eqref{eq:extremevalue}. Finally the parameters in AFT models $(\boldsymbol{b}, \sigma)$ can be estimated by maximizing the above likelihood function. 

\subsection{Calculation of the Delta Effect}
Once we learn the parameter estimation $(\hat{\boldsymbol{b}}, \hat{\sigma})$ from model training, we can calculate the delta effect for user $i$ at a given time as follows,

\begin{itemize}
\item Get  all the features $\boldsymbol{X}_{0,i}$ including the state at the moment and the time since last state (i.e., badge count update) $W_{0,i}$ for member $i$.
\item Derive the new features $\boldsymbol{X}_{1,i}$ given that a notification is sent at the moment, which updates the badge count as a state feature and state interaction features. 
\item According to \eqref{eq:lambda_alpha}, calculate 

\begin{align*}
\hat{\lambda}_{0,i} = e^{- \hat{\boldsymbol{b}} \boldsymbol{X}_{0,i}/\hat{\sigma}},\quad\hat{\alpha}_0= 1/\hat{\sigma};\\
 \hat{\lambda}_{1,i} = e^{- \hat{\boldsymbol{b}} \boldsymbol{X}_{1,i}/\hat{\sigma}}, \quad\hat{\alpha}_1= 1/\hat{\sigma}.
\end{align*}

\item Apply the above values to \eqref{eq:weibulldeltapvisit}, and calculate the delta effect

\begin{equation}
\label{eq:pvisit_calculation}
\Delta F_i(W_{0,i},T)  =  e^{-\hat{\lambda}_{0,i} (T + W_{0,i})^{\hat{\alpha}_0} + \hat{\lambda}_{0,i} (W_{0,i})^{\hat{\alpha}_0} } - e^{-\hat{\lambda}_{1,i} T^{\hat{\alpha}_1}}.
\end{equation}
\end{itemize}

\section{Data collection and model training}
\label{sec:datacollection}
Collecting large-scale unbiased training data is challenging, especially in the case of observational data. We collect data at LinkedIn from hundreds of millions of users for a given week including all badging notification events delivered to users and all user visit events. For each notification event, we include 3 broad categories of features in $\boldsymbol{X_i}$.
 \begin{itemize}
\item user's profile features such as locale and network size.
\item State features such as badge count.
\item user's activity features such as user's last visit time, the number of site visits over the past week and the number of notifications received over the past week.
\end{itemize}
In addition, we also include interactions between the above features, such as interaction terms between the badge count and the profile features so that we can learn different sensitivity to the badge count from different users.

To get the response $T_i$ and censoring indicator $\delta_i$, we sort  notification events and visit events in the temporal order for each user so that we can get the next event type and next event time $T_i$. If the next event is a visit, then $\delta_i=1$; otherwise $\delta_i=0$.  Note that the next event and next visit may extend beyond the given week, and thus the following week's data may be needed  and joined accordingly. We then remove potential outliers by discarding records from users who receive too many notifications or visit too many times. Such records may come from erroneous tracking or abnormal accounts. Next, we split a week's notification data into training and test data with a ratio of 4:1. The test data are held out for evaluation in Section \ref{sec:offline-evaluation}.

We train the AFT model with the training data on  using Spark MLlib \cite{Meng:2016:MML:2946645.2946679} and obtain $\hat{\boldsymbol{b}}$ and $\hat{\sigma}$. Parameters in  the conditional Weibull distribution can be calculated as $\hat{\alpha} = 1/\hat{\sigma}$,  $\hat{\lambda}_i = e^{-\hat{\boldsymbol{b}} \boldsymbol{X}_i/\hat{\sigma}}$. 

The model we learned from training data suggests very different feature importance from that of a notification CTR model. For example, the badge count is a strong predictor and most people are more sensitive to one badge count increase when the badge count is low and become indifferent when the badge count is high. On the other hand, the badge count, the time after the last notification sent are usually not strong signals for a notification CTR model based on our previous experience. The two models can be complementary to each other in a MOO setup described in Section \ref{sec:mooproblem}, since they seem to capture different aspects of notifications.

The $\hat{\sigma}$ we learned  is greater than $1$, so we have $\hat{\alpha}  \in (0,1)$, suggesting that  the $ \Delta F_i(W_{0,i},T) $  in \eqref{eq:deltapvisit} is increasing in $W_{0}$ according to Lemma \ref{lemma:decreasing}. This is aligned with our intuition that the longer time spacing we have from the previous notification send time, the more incentive we have to send another notification.

The model we train also suggests that the marginal effect on user engagement diminishes as the badge count increases. The interaction between badge count and user features are significant, meaning different users have different sensitivity to badging.

\section{Applications and Thresholds}
In this section, we show how our model can be leveraged by different notification decision systems. 
\subsection{Notification MOO Problems}
\label{sec:mooproblem}

The model works well with notification MOO problems as a utility function. Consider a typical example where we have notifications available to send to $N$ users and we would like to maximize the total engagement gains while increasing the total notification clicks and controlling the send volume. Let $y_i$ be the decision variable for notification $M_i$ with 1 indicating send and 0 not send. $\Delta F_i(W_{0,i},T)$ is the predicted session gain, where $W_{0,i}$ is the time since last badge update and $T$ is the prediction time window we are interested in, e.g., the next 24 hours. Assuming we have another model that predicts the probability of a click $P_i(click)$ for a notification available for user $i$ given it is sent, we can formulate a MOO problem,

\begin{equation}
\label{eq:vo_moo}
\begin{aligned}
&\text{Maximize}  &&\sum_{i=1}^{N} \Delta F_i(W_{0,i},T) y_i&\\
& \text{subject to}  &&  \sum_{i=1}^{N} P_i(click) y_i \geq C_{click},  \\
&&&  \sum_{i=1}^{N}  y_i \leq C_{send},\\
&&& 0 \leq y_i \leq 1.
\end{aligned}
\end{equation}
The objective above is to maximize user visits due to notifications, which is quantified by $\Delta F_i(W_{0,i},T)$ if notification $i$ is sent. The first constraint requires the total number of clicks on notifications to be greater than or equal to $C_{click}$, thus ensuring that the notifications sent are relevant to users. The second requires the total number of notifications sent to be less than or equal to $C_{send}$, thus controlling the send volume to avoid notification overload. 

By considering the duality of the linear programming problem, the resulting decision rule would be 
\begin{equation}
\label{eq:rule_moo}
y_i = 1  \iff  \Delta F_i(W_{0,i},T) + \kappa_1 P_i(click) > \kappa_2,
\end{equation}
where $\kappa_1$ and $\kappa_2$ correspond to dual variables for the first two constraints. The decision rule is a global threshold of $\kappa_2$ across all users on a linear combination of engagement effect  $\Delta F_i(W_{0,i},T)$ and notification quality $P_i(click)$. Similar volume optimization problems can be found in  \cite{gupta2016email,gupta2017optimizing} for emails.

\subsection{Delivery Time Optimization (DTO)}
\label{sec:dto}
Mobile notifications are time-sensitive. Sending notifications at a better timing may increase user engagement and improve user experience. The major advantage of our model is to add a utility to evaluate along the time dimension through two channels. The first one is real-time features in the model itself, such as current badge count. The other is the time since last badge update $W$, which would affect the calculation of $\Delta F_i(W_{0,i},T)$. Under the Weibull distribution assumption, $\Delta F_i(W_{0,i},T)$ is increasing in $W$ according to Lemma \ref{lemma:decreasing}, which means we have less incentive to send a notification if we already sent one shortly before and more if we have not sent one in a long time. This makes the model effective in DTO and notification spacing. 
%decouple whether to send
A straightforward strategy to achieve this is to send a notification to a user $i$ only when $\Delta F_i(W_{0,i},T)$ is above a certain threshold,
\begin{equation}
\label{eq:dto_threshold}
 \Delta F_i(W_{0,i},T)  > \kappa.
\end{equation}
In practice, we find that a modification below can improve the performance in some cases when we optimize user engagement,
\begin{equation}
\label{eq:dto_ratio}
 \frac{\Delta F_i(W_{0,i},T)} {P_i(\text{visit|not send})} > \kappa.
\end{equation}
where $P_i(\text{visit|not send})$ is defined in \eqref{eq:pvisitneg} for user $i$. The latter decision rule \eqref{eq:dto_ratio} can be viewed as a personalized version of \eqref{eq:dto_threshold}, where the personalization is based on a per-member constraint on the number of notification sends.
%This could be due to the implicit constraints that we do not send too many notifications per user and such per user constraints usually lead to personalized threshold instead of a global threshold.  The decision rule in \eqref{eq:dto_ratio} can be viewed as a personalized version of \eqref{eq:dto_threshold} based on $P_i(\text{visit|not send})$ as a baseline probability.  Section \ref{sec:usecase_dto} gives a deployed use case based on this strategy.

\section{Offline evaluation}
\label{sec:offline-evaluation}
We compare our proposed survival-based approach with the conventional baseline logistic regression model. While there are potentially more accurate baseline models such as tree models and deep models, survival models can also be extended beyond a linear structure \cite{ishwaran2008random,pmlr-v56-Ranganath16}. Such a comparison isolates the impact of data censoring and the survival approach from that of feature engineering.  For any given  time frame $T$, we train a logistic regression with the same set of features including their interactions $\boldsymbol{X}_{1,i}$ and a response of whether a user's visit occurs within $T$ after the notification is delivered. 
One advantage of our formulation over a classification task  is that the same model can be used to predict a user's probability of visiting given any time frame $T$ through a Weibull distribution  $F(T;\lambda, \alpha)$ in \eqref{eq:weibullcdf}. Therefore, the same model  can be deployed in different applications, where the prediction windows are chosen differently. On the other hand, we need to train an individual logistic regression model for each different $T$ since the response variables are different. 

To evaluate the prediction performance, we calculate the area under the receiver operating characteristic curve (AUC) for selected $T$ values as binary classification problems. For the AFT model, we calculate $F(T;\hat{\lambda}, \hat{\alpha})$ in \eqref{eq:weibullcdf} to be used in the same way as the logistic prediction for the AUC. Figure \ref{fig:auc} shows how our model compares with the baseline model in terms of the AUCs as a function of the prediction time window $T$. The shorter the time window is, the harder the prediction as a binary classification is, since the randomness of the users' engagement behavior tends to be  dominating in the short term. At $4$ hours prediction window, the model already gives a reasonable AUC of about $0.74$ while the logistic regression model only gives $0.58$. At $24$ hours for daily engagement prediction, the model gives an AUC as high as $0.85$ while the baseline model reaches $0.73$. Interestingly, as we further increase the prediction window, the AUC of the logistic regression model starts to fall while our model reaches $0.89$ at $48$ hours. The decline of the logistic regression could come from bias introduced by attributing a visit event to multiple notification events within the time window $T$. This bias becomes more severe as the time window $T$ increases and likely covers more notifications. The AFT model, on the other hand, avoids such bias by correctly incorporating information from both censored and uncensored observations.

The results show that handling censoring properly is very crucial to mobile notification data. In addition, our model achieved great flexibility in $T$ and superior prediction power compared with the logistic classification models at every given $T$.

\begin{figure}[h]
\includegraphics[width=\linewidth]{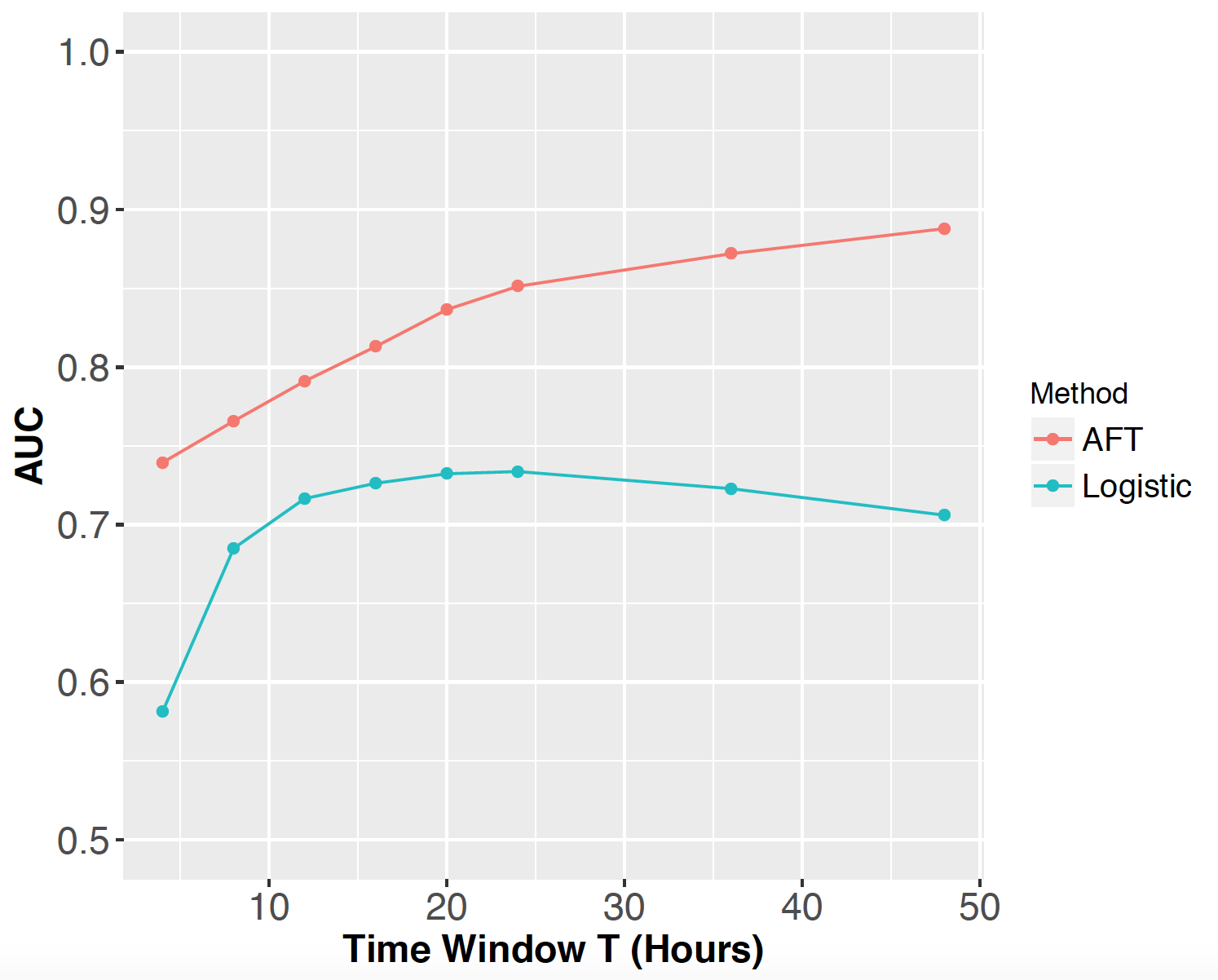}
\caption{\label{fig:auc} AUC as a function of T} 
\end{figure}

\begin{figure*}
  \includegraphics[width =0.9 \linewidth]{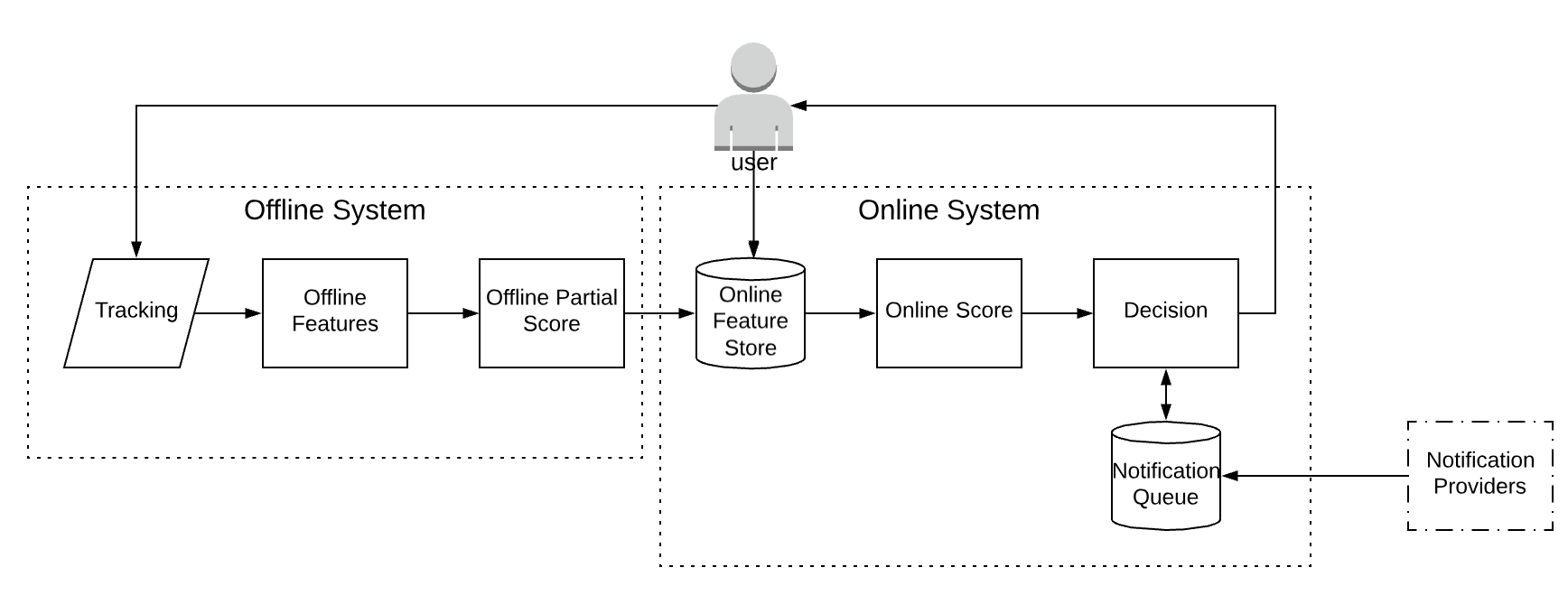}
\vspace{-0.1cm}
  \caption{System architecture}
  \label{fig:infra}
\end{figure*}

\section{Online use case and experiment}
\label{sec:usecase}

In this section, we present a case study deployed at LinkedIn to show how we improve our decision making for mobile notifications with our model.

%\begin{figure}[h]
%\includegraphics[width=0.6\linewidth]{figures/birthday.eps}
%\caption{\label{fig:birthday} An Example Birthday Notification} 
%\end{figure}

\subsection{Delivery Time Optimization for Less Time-sensitive Notifications}
\label{sec:usecase_dto}
Social network services send both time-sensitive and less-time-sensitive notifications to users. Time-sensitive ones are usually triggered by user-to-user messages or connections' activities such as sharing an article. These notifications need to be sent immediately when triggered to keep users well-informed. There are also types of notifications which are less time-sensitive. For example, birthday notifications reminding a connection's upcoming birthday can be sent the day of the birthday or several days ahead. In this use case, the less-time-sensitive notifications are first filtered based on a click-through-rate (CTR) prediction model thus dropping notifications with low predicted CTR to ensure high notification quality. The filtered are queued to be then sent within a valid send time window for each individual notification. The send time window ranges from a few hours to a few days depending on the nature of the notification. In this application, ``when to send'' is decoupled from ``whether to send'' since the latter decision is already made at the filtering stage. This makes a good use case of delivery time optimization described in Section \ref{sec:dto}.

In this application, we apply the decision rule in \eqref{eq:dto_ratio}, where $T$ is set to be 4 hours and $\kappa$ is chosen from offline analysis and online tuning to optimize the performance. For comparison, we set up a control treatment, in which notifications are sent immediately when available, and a baseline treatment, in which we send a notification to a user only if their badge counts are less than or equal to 1. For users who have notifications in the queue, we evaluate send decisions every $4$ hours.

\subsection{Online Experiment and Results}
Table \ref{tab:abtest} shows the A/B test experiment results comparing the DTO based on our model with the control and baseline models described above. We are mostly interested in user engagement and notification interactions, which can be characterized by the following metrics.
\begin{itemize}
\item \emph{Sessions}: A session is a collection of full page views made by a single user on the same device type. Two sessions are separated by $15$ minutes of zero activity. 
%This metric counts the number of total site-wide sessions, which is what the company reports externally.
\item \emph{Engaged Feed Sessions}: This metric counts the number of sessions where the user engaged with the newsfeed (either by interacting with feed updates, or by viewing at least $10$ feed updates).
\item \emph{Notification Sessions}: This metric counts the number of sessions in which the user viewed or interacted with the notification page.
\item \emph{Notification Daily Unique Send CTR}: This metric measures the average click-through-rate of notifications sent to a user in a day.
\end{itemize}

The experiment was tested over a full week and the numbers in the table are all %site-wide and 
statistically significant. Compared with the control, which is basically no DTO, our model increased the total sessions by $1.86\%$, notification sessions by $6.19\%$ and engaged feed sessions by $1.78\%$. The higher boost in notification sessions was expected since we are optimizing notification send time. The roughly proportional gain in engaged feed sessions suggests that the additional sessions are of similar quality to existing ones. In addition, the notification daily unique send CTR was increased by 2.51\% against control, suggesting notifications were delivered at better timing resulting in increased user engagement. Compared with the baseline model, the proposed model showed healthy gains in all four metrics. One interesting observation is that the increase in notification daily unique send CTR (+4.48\%) is higher than the comparison with the control (+2.51\%). This suggests that although the badge count baseline model increases user engagement, it reduces the CTR compared with the control, implying that it may not be a desirable user experience.

\begin{table}[!h]
\caption{Online A/B results for delivery time optimization}
\centering
  \begin{tabular}{ | l | c | c| }
    \hline
    Metric & vs. Control  & vs. Baseline\\
    \hline
    \hline
   Sessions & + 1.86\% & +0.67\%\\
     \hline
    Engaged Feed Sessions  & + 1.78\% & +0.69\%\\
    \hline
     Notification Sessions  & +6.19\% & +1.51\%\\
    \hline
    Notification Daily Unique Send CTR & +2.51\% & +4.48\%\\
    \hline
  \end{tabular}
\label{tab:abtest}
\end{table}

\subsection{System Architecture}

We outline a design of a notification decision system using the state transition model in Figure \ref{fig:infra}. Since the model takes a few real-time features (e.g., current badge count, time since last badge count update) as important signals, having an online scoring system is ideal for model performance. To avoid maintaining all features in an online database, we include an offline component for more static features, such as user profile features. In this offline component, offline features are retrieved from tracking data on our HDFS system and a partial score is calculated based on the trained model coefficients. We push the partial scores to an online feature store daily through Apache Kafka \cite{kreps2011kafka}. The online component maintains realtime features and make realtime decisions based on the real-time $\Delta F_i(W_{0,i},T)$ score.

\section{Discussion}
To our best knowledge, this is the first work on probabilistic modeling of interactions between mobile notifications and user engagement at scale. We develop a state transition model and derive a delta effect to measure the effectiveness of a notification. With a common existence of censoring in observational mobile notification data, we estimate the delta effect through an AFT regression with a Weibull distribution. The prediction from this survival regression is both flexible to apply and superior in prediction accuracy compared to baseline models with the same feature set.

Our state transition model is generalizable and can have broader applications. While we focus on modeling the badging notifications, our model is applicable for all types of mobile notifications. For example, UI push notifications can be modeled with a distribution possibly different from a Weibull distribution. 

We consider a user's visit as a reward to a mobile notification. However, the reward can be generalized to a user's purchase event for on-line shopping apps such as Amazon  or a user's content creation event for question-and-answer apps such as Quora. In the cases where data censoring is a major concern for modeling mobile notifications, we provide a general framework to evaluate the effectiveness of a notification towards driving a reward.

 \section*{Acknowledgement}
 We would sincerely like to thank Rupesh Gupta, Matthew Walker, Kinjal Basu, Yan Gao, Haoyu Wang, Myunghwan Kim, Guangde Chen,  Ajith Muralidharan for their detailed and insightful feedback during the development of this model.

%%%%%%%%%%%%%%%%%%%%%%%%%%%%

\bibliographystyle{ACM-Reference-Format}
\bibliography{pvisit} 

\end{document}